\documentclass[letterpaper, 10 pt, conference]{ieeeconf} 
\IEEEoverridecommandlockouts  
\title{\LARGE \bf
Continuous Motion Planning with Temporal Logic Specifications \\  using Deep Neural Networks
}

\author{Chuanzheng Wang, Yinan Li,  Stephen L.\ Smith, Jun Liu
\thanks{Chuanzheng Wang, Yinan Li, and Jun Liu are with the Department of Applied Mathematics,        
University of Waterloo, Waterloo, Ontario, Canada, 
        {\tt\small \{cz.wang, yinan.li, j.liu\}@uwaterloo.ca}}
\thanks{Stephen L. Smith is with the Department of Electrical and Computer Engineering, 
University of Waterloo, Waterloo, Ontario, Canada, 
        {\tt\small stephen.smith@uwaterloo.ca}}
}

\usepackage{amsmath,amssymb,amstext}
\usepackage{cite}
 \newtheorem{problem}{Problem}
 \usepackage{algorithm,algorithmic}
 \usepackage{graphicx}
 \usepackage{subcaption}
 \usepackage{booktabs}
 \usepackage{mwe}
 \renewcommand{\epsilon}{\varepsilon}
 \renewcommand{\theta}{\vartheta}
 \renewcommand{\kappa}{\varkappa}
 \renewcommand{\rho}{\varrho} 
 \renewcommand{\phi}{\varphi}
 \newcommand{\sv}{\;\vert\;}
 \newcommand{\set}[1]{\left\{#1\right\}}
 
 \newcommand{\word}{\mathbf{w}}
 \newcommand{\until}{\mathbf{U}}
 \newcommand{\always}{\Box}
 \newcommand{\eventually}{\lozenge}
 
 \newcommand{\A}{\mathcal{A}}
 \newcommand{\B}{\mathcal{B}}
 \newcommand{\F}{\mathcal{F}}
 \newcommand{\Q}{\mathcal{Q}}
 
 \newcommand{\Real}{\mathbb{R}}
 \newcommand{\outprops}{\text{OutProps}}
 \newcommand{\outedges}{\text{OutEdges}}
 
 \newtheorem{definition}{Definition}
 
 \newtheorem{thm}{Theorem}
 \newtheorem{rem}{Remark}
\usepackage[labelformat=simple]{subcaption}

\newcommand{\algorithmicbreak}{\textbf{break}}
\newcommand{\BREAK}{\STATE \algorithmicbreak}
\usepackage{xcolor}

\usepackage{array}
\newcolumntype{P}[1]{>{\centering\arraybackslash}p{#1}}

\begin{document}

\maketitle
\thispagestyle{empty}
\pagestyle{empty}

\begin{abstract}

In this paper, we propose a model-free reinforcement learning method to synthesize control policies for motion planning problems with continuous states and actions. The robot is modelled as a labeled discrete-time Markov decision process (MDP) with continuous state and action spaces. Linear temporal logics (LTL) are used to specify high-level tasks. We then train deep neural networks to approximate the value function and policy using an actor-critic reinforcement learning method. The LTL specification is converted into an annotated limit-deterministic B\"uchi automaton (LDBA) for continuously shaping the reward so that dense rewards are available during training. A na\"ive way of solving a motion planning problem with LTL specifications using reinforcement learning is to sample a trajectory and then assign a  high reward for training if the trajectory satisfies the entire LTL formula. However, the sampling complexity needed to find such a trajectory is too high when we have a complex LTL formula for continuous state and action spaces. As a result, it is very unlikely that we get enough reward for training if all sample trajectories start from the initial state in the automata. In this paper, we propose a method that samples not only an initial state from the state space, but also an arbitrary state in the automata at the beginning of each training episode. We test our algorithm in simulation using a car-like robot and find out that our method can learn policies for different working configurations and LTL specifications successfully. 

\end{abstract}

\section{INTRODUCTION}

Traditionally, motion planning problems consider generating a trajectory for reaching a specific target while avoiding obstacles \cite{karaman2011sampling}. However, real-world applications often require more complex tasks than simply reaching a target. As a result, recent motion planning problems consider a class of high-level complex specifications that can be used to describe a richer class of tasks. A branch of planning approaches has been proposed recently that describes high-level tasks like reaching a sequence of goals or ordering a set of events using formal languages such as linear temporal logic (LTL)~\cite{loizou2004automatic}. As a simple example, the task of \textit{reaching region A and then reaching region B} can be easily expressed as an LTL formula. To deal with LTL specifications, an approach for dealing with a point-mass robot model has been proposed in \cite{fainekos2009temporal}. A control synthesis technique with receding horizon control has been proposed in \cite{wongpiromsarn2009receding} to handle a linear robot model. The approach in \cite{bhatia2010sampling} uses sampling-based method to deal with nonlinear dynamic robot models with LTL specifications. However, this method suffers from the \textit{curse of dimensionality} limiting its use to low-dimensional system models. 

Reinforcement learning has achieved great success in the past decades both in terms of theoretical results~\cite{sutton2000policy} and application~\cite{silver2016mastering,silver2017mastering}. It is a way of learning the best actions for a Markov decision process (MDP) by interacting with the environment~\cite{sutton1998introduction}. It is efficient in solving problems of complex systems with or without knowing a model \cite{sutton2018reinforcement}. Early works were mainly based on $Q$-learning \cite{watkins1989learning} and policy gradient methods \cite{sutton2000policy}. The actor-critic algorithm \cite{peters2005natural} is also widely used with two components, namely an actor and a critic. The actor is used as the policy, which tells the system what action should be taken at each state, and the critic is used to approximate the state-action value function.  Modern reinforcement learning methods take advantage of deep neural networks to solve problems with large state and action spaces. A deep $Q$-network (DQN) \cite{mnih2013playing} uses a deep neural network to approximate state-action values and learns an implicit control policy by improving this $Q$-network. In \cite{silver2014deterministic}, a deterministic policy gradient method is proposed with better time efficiency and consequently the deep deterministic policy gradient method (DDPG) \cite{lillicrap2015continuous} leverages this idea of a deterministic policy and uses two deep neural networks, an actor network and a critic network, to solve problems of continuous state and action spaces. 

Reinforcement learning algorithms have been applied to solve model-free robotic control problems with temporal logic specifications. In \cite{sadigh2014learning}, a $Q$-learning method is used to solve an MDP problem with LTL specifications. The temporal logical formula is transformed into a deterministic Rabin automaton (DRA) and a real-valued reward function is designed in order to satisfy complex requirements. In \cite{gao2019reduced}, a reduced variance deep $Q$-learning method is used to approximate the state-action values of the product MDP with the help of deep neural networks. Another branch of methods convert the LTL formula into a limit-deterministic B\"uchi automaton (LDBA) and a synchronous reward function is designed based on the acceptance condition of the LDBA as in \cite{hasanbeig2019reinforcement} and \cite{hahn2019omega}. The authors in \cite{hasanbeig2018logically} use neural fitted $Q$-iteration to solve systems with continuous state. In \cite{oura2020reinforcement}, the limit-deterministic generalized B\"uchi automaton (LDGBA) is used to convert the LTL formula. Moreover, a continuous state space is considered in \cite{hasanbeig2019certified}. 

However, training deep networks for continuous controls is more challenging due to significantly increased sample complexity and most approaches achieve poor performance on hierarchical tasks, even with extensive hyperparameter tuning \cite{duan2016benchmarking}. This is because the reward function is so sparse if we only have a terminal reward when the accepting conditions are satisfied. The authors in \cite{li2018policy} use a time-varying linear Gaussian process to describe the policy and the policy updated by maximizing the robustness function in each step. In \cite{lavaei2020formal}, the authors use model-free learning to synthesize controllers for finite time horizon for systems with continuous state and action space. The authors in \cite{yuan2019modular} also propose a training scheme for continuous state and action space but using one neural network for each individual automaton state. This requires a large number of networks when the LTL specification is complex.  

As a result, the main contribution of this paper is that by using an annotated LDBA converted from the LTL specification and a simple idea that randomly samples from the automaton states without initializing it to a fixed initial state (as given by the translated automaton), we can effectively train deep networks to solve continuous control problems with temporal logic goals. We show in simulations that our method achieves a good performance for a nonlinear robot model with complex LTL specifications.

\section{Preliminary and Problem Definition}
\subsection{Linear Temporal Logic}\label{sec:ldba}
Linear temporal logic (LTL) formulas are composed over a set of atomic propositions $AP$ by the following syntax:
\begin{equation}\label{eq:ltl}
\varphi::=\textbf{true}\sv p\sv \neg\varphi\sv \varphi\wedge\varphi\sv \varphi\until\varphi\sv \bigcirc\varphi,
\end{equation}
where $p\in AP$ is an atomic proposition, \textbf{true}, {\em negation\/} ($\neg$), {\em conjunction\/} ($\wedge$) are propositional logic operators and {\em next\/} ($\bigcirc$), {\em until\/} ($\until$) are temporal operators.

Other propositional logic operators such as \textbf{false}, {\em disjunction} ($\vee$), {\em implication} ($\rightarrow$), and temporal operators {\em always} ($\always$), {\em eventually} ($\eventually$) can be derived based on the ones in (\ref{eq:ltl}). A sequence of symbols in $\Sigma=2^{AP}$ is called a \emph{word}. We denote by $\word\vDash\varphi$ if the word $\word$ satisfies the LTL formula $\varphi$. Details on syntax and semantics of LTL can be found in~\cite{baier2008principles}.

For probabilistic systems such as Markov decision processes, it is sufficient to use a limit-deterministic B\"uchi automaton (LDBA) over the set of symbols $\Sigma$, which are deterministic in the limit, to guide the verification or control synthesis with respect to an LTL formula. For any LTL formula $\varphi$, there exists an equivalent LDBA that accepts exactly the words described by $\varphi$ \cite{Sickert2016}.

In the current paper, we use transitions-based LDBA, since it is often of smaller size than its state-based version.  We begin by defining a transition-based  B\"uchi automaton and then give a formal definition of an LDBA.

\begin{definition}
    A \emph{transition-based generalized B\"uchi automaton} (TGBA) is a tuple $\A=(\Q,\Sigma,\delta,q_0,\F)$, where $\Q$ is a set of states, $\Sigma$ is a finite alphabet, $\delta:\Q\times\Sigma\to 2^{\Q}$ is the state transition function, $q_0\in \Q$ is the initial state, and $\F=\set{F_1,\cdots,F_m}$ with $F_i\subseteq \Q\times\Sigma\times\Q$ ($i\in\set{1,\cdots,m}, m\geq 1$) is a set of accepting conditions.
\end{definition}

A \emph{run} of a TGBA $\mathcal{A}$ under an input word $\word=\sigma_0\sigma_1\cdots$ is an infinite sequence of transitions in $\Q\times\Sigma\times\Q$, denoted by $\xi=(v_0,\sigma_0,v_1)(v_1,\sigma_1,v_2)\cdots$, that satisfies $v_{i+1}\in \delta(v_i,\sigma_i)$ and $v_i\in \Q$ for all $i\in\mathbb{N}$. Let $\xi[i]=(v_i,\sigma_i,v_{i+1})$, $\outprops(q)=\set{\sigma\in\Sigma\sv \exists q'\in\Q\, \text{s.t.}\, q'\in\delta(q,\sigma)}$ and $\outedges(q)=\set{(q,\sigma,q')\sv \exists q'\in\Q,\sigma\in\Sigma\; \text{s.t.}\; q'\in\delta(q,\sigma)}$. Denote by $(q,\sigma,q')$ the transition between $q,q'\in\Q$ under the input $\sigma\in\Sigma$. A word $\word$ is accepted by $\A$ if there exists a run $\xi$ such that ${\rm Inf}(\xi)\cap F_j\neq\emptyset$ for all $j\in\set{1,\cdots, m}$, where ${\rm Inf}(\xi)=\set{(v,\sigma,v')\in\Q\times\Sigma\times\Q\sv \forall i, \exists j>i,\; \text{s.t.}\;(v,\sigma,v')=\xi[j]}$ is the set of transitions that occur infinitely often during the run $\xi$.

\begin{definition}\label{def:ldba}
    A TGBA $\A=(\Q,\Sigma\cup\set{\varepsilon},\delta,q_0,\F)$ is a \emph{limit-deterministic B\"uchi automaton} (LDBA) if $\Q=\Q_N\cup\Q_D$, $\Q_N\cap\Q_D=\emptyset$, and
    \begin{itemize}
        \item $\delta(q,\sigma)\subseteq \Q_D$ and $|\delta(q,\sigma)|=1$ for all $q\in\Q_D$ and $\sigma\in\Sigma$,
        \item $|\delta(q,\sigma)\cap\Q_N|=1$ for all $q\in\Q_N$ and $\sigma\in\Sigma$,
        \item $F\subseteq\Q_D\times\Sigma\times\Q_D$ for all $F\in\F$.
        \item if $q'\in\delta(q,\sigma)$ for $\sigma\in\Sigma$, $q\in\Q_N$ and $q'\in\Q_D$, then $\sigma=\epsilon$.
    \end{itemize}
\end{definition}
The transitions from $\Q_N$ to $\Q_D$ are called \emph{$\varepsilon$-transitions} defined in the last condition of Definition \ref{def:ldba}, which are taken without reading any input propositions in $\Sigma$. The correctness of taking any of the $\varepsilon$-transitions can be checked by the transitions in $Q_D$ \cite{Sickert2016}. For an LDBA $\A$, we define $\outprops(q)=\set{\sigma\in \Sigma\sv \exists q'\in\Q\; \text{s.t.}\; q'\in\delta(q,\sigma)}$.

\subsection{Labeled Markov Decision Process}
To capture the robot motion and working properties, we use a continuous labeled Markov decision process with discrete time to describe the dynamics of the robot and its interaction with the environment \cite{Sebastian2002}.
\begin{definition}\label{labelmdp}
	A continuous \emph{labeled Markov Decision Process} (MDP) is a tuple $M=\{S,A,P,R,\gamma,AP,L\}$, where $S\subseteq\Real^n$ is a continuous state space, $A\subseteq\Real^m$ is a continuous action space, $P:S\times A\to\kappa(S)$ is a transition probability kernel with $P(\cdot|s,a)$ defining the next-state distribution of taking action $a\in A$ at state $s\in S$, the function $R:S\times A\times S\to\mathbb{R}$ specifies the reward, $\gamma$ is a discount factor, $AP$ is the set of atomic propositions, and $L:S\to 2^{AP}$ is the labeling function that returns propositions that are satisfied at a state $s\in S$. Here $\kappa(S)$ denotes the set of all probability measures over $S$.
\end{definition}

The labeling function $L$ is used to assign labels from a set $AP$ of atomic propositions to each state in the state space $S$. Given a sequence of states $\mathbf{s}=s_0s_1\cdots$, a sequence of symbols $\mathbf{tr}(\mathbf{s})=L(s_0)L(s_1)\cdots$, called the \emph{trace} of $\mathbf{s}$, can be generated to verify if it meets a LTL specification $\varphi$. If $\mathbf{tr}(\mathbf{s})\vDash\varphi$, we also write $\mathbf{s}\vDash\varphi$.

\begin{definition}
	A \emph{deterministic policy} $\pi$ of a labeled MDP is a function $\pi:S\to A$ that maps a state $s\in S$ to an action $a\in A$.
\end{definition}

Given a labeled MDP, we can define the accumulated reward starting from state $s$ as
\begin{equation*}
 G(s)=\sum_{k=0}^{\infty}\gamma^{k}R(s_k,a_k,s_{k+1}\big\lvert s_0=s).
\end{equation*}

\subsection{Product MDP}
Considering a robot operating in a working space to accomplish a high-level complex task described by an LTL-equivalent LDBA $\mathcal{A}$, the mobility of the robot is captured by a labeled MDP $M$ 

defined as above. We can combine the labeled MDP and the LDBA to obtain a product MDP.
\begin{definition}
	A \emph{product} MDP between a labeled MDP $M=\{S,A,P,R,\gamma,AP,L\}$ and an LDBA $\A=(\Q,\Sigma,\delta,q_0,\F)$ is a tuple
	\[M_p=M\times\A:=\{S_p,A_p,P_p,R_p,AP,L,\F_p,\gamma\},\]where 
	\begin{itemize}
	    \item $S_p=S\times\Q$ is the set of states,
	    \item $A_p=A\cup\set{\varepsilon}$ is the set of actions,
	    
	    \item $P_p:(S\times\Q)\times A\to\kappa(S,\Q)$ is the transition probability kernel defined as
	    \[P_p(s_p'|s_p,a_p)=\begin{cases}
	    P(s'|s,a), & q'=\delta(q,L(s)),a_p\in A,\\
	    1, & s=s',a_p=\varepsilon,\\ 
	    0, & \text{otherwise},
	    \end{cases}\]
	    for all $s_p=(s,q), s_p'=(s',q')\in S_p$,
	    \item $R_p:S_p\times A_p\times S_p\to\mathbb{R}$ is the reward function, and
	    \item $\F_p=\set{F_p^1,\cdots,F_p^m}$ ($m\geq 1$), where $F_p^i=\{((s,q),a,(s',q'))\in S_p\times A_p\times S_p\sv (q,a,q')\in F_i\}$ for all $i\in\set{1,\cdots,m}$ is a set of accepting conditions.
	\end{itemize}

\end{definition}

Likewise, a \textit{run} of a product MDP $M_p=M\times\A$  is an infinite sequence of transitions of the form $$\xi=((s_0,q_0),a_0,(s_1,q_1))((s_1,q_1),a_1,(s_2,q_2))\cdots,$$
where $((s_i,q_i),a_i,(s_{i+1},q_{i+1}))\in S_p\times A_p \times S_p$. 
We say that $\xi$ is an \textit{accepted run}, denoted by $\xi\vDash\A$, if ${\rm Inf}(\xi)\cap F_p^i\neq\emptyset$ for all $i\in\set{1,\cdots,m}$, where ${\rm Inf}(\xi)$ is the set of transitions that occur infinitely often in $\xi$.

\subsection{Problem Formulation}

We consider the problem in which a robot and its environment are modelled as an MDP $M$, and the robot task is specified as an LTL formula $\phi$.  Given an initial state $s_0$, we define the probability of an MDP $M$ satisfying $\phi$ under a policy $\pi$ from $s_0$  as
\begin{equation*}
    \mathbb{P}_{\pi}^M(s_0\vDash\phi):=\mathbb{P}(\mathbf{s} \in \mathcal{P}_{\pi}^M|\mathbf{s}\vDash\phi,\,\mathbf{s}=s_0s_1\cdots),
\end{equation*}
where $\mathcal{P}_{\pi}^M$ is the set of all infinite sequences of states of the MDP that are induced from the policy $\pi$. We say a formula $\phi$ is satisfied by a policy $\pi$ at $s_0$ if $\mathbb{P}_{\pi}^M(s_0\vDash\phi)>0$. If such a policy exists, we say that $\phi$ is satisfiable at $s_0$.  

Then the problem we address in this paper is as follows.

\begin{problem}\label{prob:ltl+mdp}
    Given a continuous labeled MDP $M=\{S,A,P,R,\gamma,AP,L\}$ and an LTL specification $\phi$, find a policy $\pi^*$ such that $\phi$ is satisfied by $\pi^*$ for each $s_0\in S$ such that $\phi$ is satisfiable at $s_0$. 
\end{problem}

As we have seen in Section \ref{sec:ldba}, an LTL formula $\varphi$ can be translated into an LDBA $\A_\varphi$. Therefore, solving Problem \ref{prob:ltl+mdp} is equivalent to solving the following control problem for the corresponding product MDP $M_p$ of the given MDP $M$ and $\A_\varphi$ \cite{Courcoubetis1995}.

Consider the product MDP $M_p=M\times A_\varphi$. We say that a policy $\pi$ for $M_p$ satisfies $\A_\varphi$ at $(s_0,q_0)$, where $s_0\in S$ and the initial state $q_0$ of $\A_\varphi$, if $\mathbb{P}_{\pi}^{M_p}((s_0,q_0)\vDash\A_\phi)>0$, where
    \begin{align*}
    &\mathbb{P}_{\pi}^{M_p}((s_0,q_0)\vDash\A_\varphi):=\\
    &\mathbb{P}(\xi \in \mathcal{P}_{\pi}^{M_p}|\xi\vDash\A_\phi,\,\xi=((s_0,q_0),a_0,(s_1,q_1))\cdots),
\end{align*}
and $\mathcal{P}_{\pi}^{M_p}$ is the set of all runs of the product MDP that are induced from the policy $\pi$. If such a policy exists, we say that $\A_\varphi$ is satisfiable at $(s_0,q_0)$. 

\begin{problem}\label{prob2}
    Given a continuous labeled MDP $M=\{S,A,P,R,\gamma,AP,L\}$ and an LDBA $\A_\varphi$ translated from an LTL specification $\varphi$, find a policy $\pi^*$ for the product MDP $M_p=M\times A_\varphi$ such that $A_\varphi$ is satisfied by $\pi^*$ for each $s_0\in S$ and the initial state $q_0$ of $\A_\varphi$ such that $\A_\varphi$ is satisfiable. 
\end{problem}

\section{Reinforcement Learning Method}

For an MDP, the value of a state $s$ under a policy $\pi$, denoted as $v_{\pi}(s)$, is the expected return when starting from $s$ and following $\pi$ thereafter. We define $v_{\pi}(s)$ formally as 
\begin{align*}
    v_{\pi}(s)=\mathbb{E}_{\pi}\left[\sum_{k=0}^{\infty}\gamma^{k}R(s_k, a_k, s_{k+1})\Big|s_0=s\right]
\end{align*}
for all $s\in S$. Similarly, the $Q$-value $Q_{\pi}(s,a)$ for a policy $\pi$ is the value of taking action $a$ at state $s$ and following $\pi$ thereafter. It is defined as
\begin{equation*}
    Q_{\pi}(s,a)=\mathbb{E}_{\pi}\left[\sum_{k=0}^{\infty}\gamma^{k}R(s_k, a_k, s_{k+1})\Big|s_0=s,a_0=a\right].
\end{equation*}
 An optimal state-action value $Q^*(s,a)=\max_{\pi}Q_{\pi}(s,a)$ is the maximum state-action value achieved by any policy for state $s$ and action $a$. $Q$-learning\cite{watkins1989learning} is a method of finding the optimal strategy for an MDP. It learns the state-action value $Q(s,a)$ by using the update rule $Q(s,a)\gets Q(s,a)+\beta[R(s,a,s')+\gamma Q(s',a')-Q(s,a)]$, where $\beta\in[0,1]$ is a learning rate and $s'$ is the next-state of taking action $a$ at state $s$ and $a'$ is the best action at $s'$ according to the current $Q$-values.

The deep deterministic policy gradient method (DDPG) \cite{silver2014deterministic} introduces a parameterized function  $\pi^{\theta}:S\to A$, called an actor, to represent the policy using a deep neural network. A critic $Q^{\omega}(s,a)$, which also uses a deep neural network with a parametric vector $\omega$, is used to represent the action-value function. The critic $Q^{\omega}(s,a)$ is updated by minimizing the following loss function:
	\begin{equation*}
			L(\omega)=\mathbb{E}_{s\sim\rho^{\pi^{\theta}}}\bigg[(y-Q^{\omega}(s,\pi^{\theta}(s)))^2\bigg],
	\end{equation*}
where $y=R(s, a)+\gamma Q^{\omega}(s',a')$ such that $a'=\pi^{\theta}(s')$ and $\rho^{\pi^{\theta}}$ is the state distribution under policy $\pi^{\theta}$.
The objective function of the deterministic policy defined as
\begin{equation*}
J(\pi^{\theta})=\int_{S}\rho^{\pi}(s)R(s,\pi^{\theta}(s))ds=\mathbb{E}_{s\sim\rho^{\pi}}[R(s,\pi^{\theta}(s))].
\end{equation*}
is used to evaluate the performance of a policy for the MDP. According to the \emph{Deterministic Policy Gradient Theorem} \cite{silver2014deterministic}, 
\begin{align*}
\begin{split}
\nabla_{\theta}J(\pi^{\theta})&=\int_{S}\rho^{\pi}(s)\nabla_{\theta}\pi^{\theta}(s)\nabla_{a}Q^{\pi}(s,a)|_{a=\pi^{\theta}(s)}ds\\
&=\mathbb{E}_{s\sim\rho^{\pi}}[\nabla_{\theta}\pi^{\theta}(s)\nabla_{a}Q^{\pi}(s,a)|_{a=\pi^{\theta}(s)}],
\end{split}
\end{align*}
and the deterministic policy can be updated by
\begin{equation*}
\theta_{k+1}=\theta_{k}+\alpha\mathbb{E}_{s\sim\rho^{\pi^{\theta_{k}}}}[\nabla_{\theta_k}Q^{\pi^{\theta_k}}(s,\pi^{\theta_k}(s))],
\end{equation*}
where $\alpha\in[0,1]$ is a learning rate.
By applying the chain rule,
\begin{equation*}
\theta_{k+1}=\theta_{k}+\alpha\mathbb{E}_{s\sim\rho^{\pi^{\theta_k}}}[\nabla_{\theta_k}\pi^{\theta_k}(s)\nabla_a Q^{\pi^{\theta_k}}(s,a)|_{a=\pi^{\theta_k}(s)}].
\end{equation*}
It is stated in \cite{lillicrap2015continuous} that we can use the critic to approximate the objective function of the policy, which means $Q^{\pi^{\theta_k}}\approx Q^{\omega}$. As a result, we can update the parameters of the actor using
\begin{equation}\label{eq:policy}
    \theta_{k+1}=\theta_{k}+\alpha\mathbb{E}_{s\sim\rho^{\pi^{\theta_k}}}[\nabla_{\theta_k}\pi^{\theta_k}(s)\nabla_a Q^{\omega}(s,a)|_{a=\pi^{\theta_k}(s)}].
\end{equation}
The DDPG method moves the parameter vector $\theta$ greedily in the direction of the gradient of $Q$ and is more efficient in solving MDP problems with continuous state and action spaces. As a result, we propose a learning method to solve motion planning problems with LTL specifications based on DDPG. 

\section{Reinforcement Learning with LDBA-Guided Reward Shaping}
	
In this section, we introduce our method of solving a continuous state and action MDP with LTL specifications using deep reinforcement learning. The LTL specification is transformed into an annotated LDBA and a reward function is defined on the annotated LDBA for reward shaping in order to training the networks with dense reward.

\subsection{Reward Shaping}

Our definition of the reward function for the product MDP $M_p$ depends on an annotated LDBA defined as follows.
\begin{definition}\label{def:annotated}
    An \emph{annotated} LDBA $(\A,\B)$ is an LDBA  $\A=(\Q,\Sigma,\delta,q_0,\F)$ augmented by $\B=\set{b_1,\ldots,b_m}$ ($m\geq 1$), where $\F=\set{F_1,\cdots,F_m}$ and $b_i:\Q\times\Sigma\times\Q\to\set{0,1}$ ($i\in\set{1,\ldots,m}$) is a function assigning 0 or 1 to all the edges of $\A$ according to the following rules:
    \begin{align*}
    b_i(q,\sigma,q')=
        \begin{cases}
        1 & (q,\sigma,q')\in F_i,\\
        0 & \text{otherwise}.
        \end{cases}
    \end{align*}
\end{definition}

For any $i\in\set{1,\ldots,m}$, the map $b_i$, which corresponds to $F_i$, assigns 1 to all the accepting transitions in $F_i$ and 0 to all others. The set $\B$ defined above, however, only marks the accepting transitions but not the other transitions that can be taken so that the accepting transitions can happen in some future steps. In order to also identify such transitions to guide the design of the reward function of the product MDP, we provide the following Algorithm~\ref{alg:bmap} to pre-process the set $\B$ of boolean maps.

\begin{algorithm}[ht]
	\caption{Pre-process $\B$ with respect to $\A$}
	\label{alg:bmap}
	\begin{algorithmic}[1]
	\REQUIRE An annotated LDBA $(\A,\B)=(\Q,\Sigma,\delta,q_0,\F,\B)$
	\STATE Define a function $g:\Q\to\set{0,1}$
	\FOR{$i=1,\ldots,m$}
	\STATE $g(q)\leftarrow 0$ for all $q\in\Q$
	\STATE $g(q)\leftarrow 1$ if $\exists e\in\outedges(q)$ s.t. $e\in F_i$
	\REPEAT
	\FOR{$q\in\Q$ s.t. $g(q)=0$ }
	\IF{$g(q')=1$ for some $q'\in\delta(q,\sigma)$}
	\STATE $b_i(q,\sigma,q')\leftarrow 1$
	\STATE $g(q)\leftarrow 1$
	\ENDIF
	\ENDFOR
	\UNTIL{$g$ is unchanged}
	\ENDFOR
	\end{algorithmic}
\end{algorithm}

The function $g:\Q\to\set{0,1}$ in line 1 of Algorithm~\ref{alg:bmap} is defined to gradually mark every state in $\Q$ that has outgoing transitions annotated by 1. For each set $F\in\F$, the function $g$ is initialized (in line 3 and 4) to 1 for any state $q\in\Q$ that has at least one accepting outgoing transition and 0 for any other states. By using $g$, the loop from line 5 to 12 in Algorithm~\ref{alg:bmap} marks backwardly the state $q\in\Q$ with no outgoing transition marked 1 (i.e., $g(q)=0$), through which the accepting transitions can be taken. The loop terminates in a finite number of steps since the set $\Q$ of states is finite and $g$ can only be marked to 1 not $0$. After running Algorithm~\ref{alg:bmap}, for each $i\in\set{1,\cdots,m}$, the map $b_i\in\B$ marks 1 to the transitions that either are accepting in $F_i$ or can lead to the occurrence of accepting transitions in $F_i$. A state $q\in\Q$ with $g(q)=0$ after the end of $i$th for loop (for all $i\in\set{1,\cdots,m}$) is called a \emph{trap}, because accepting transitions do not occur in any run that passes through $q$.

Since any accepting run of an LDBA $\A$ should contain infinitely many transitions from each $F\in\F$, the status that whether there is at least one transition in any $F\in\F$ is taken should be tracked. For this purpose, we let $V$ be a Boolean vector of size $m\times 1$ and $V[i]$ be the $i$th element in $V$, where $k$ is the number of subsets in the accepting condition $\F$ and $i\in\set{1,\cdots,m}$. The vector $V$ is initialized to all ones and is updated according to the following rules:
\begin{itemize}
    \item If a transition in set $F_i$ is taken, then $V[i]=0$.
    \item If all elements in $V$ are 0, reset $V$ to all ones.
\end{itemize}

Now we define a function $b:\Q\times\Sigma\times\Q\to\set{0,1}$ that is updated by vector $V$ as follows:
\begin{align}\label{eq:b}
    b(q,\sigma,q')=\bigvee_{i=1}^m \left(b_i(q,\sigma,q')\wedge V[i]\right).
\end{align}
For a transition $(q,\sigma,q')$, $b(q,\sigma,q')=1$ if and only if there exists an $F_i$ that has not been visited (i.e., $V[i]=1$) and $b_i(q,\sigma,q')=1$.

Based on the above definitions, the reward function of the product MDP $M_p$ is defined as:
\begin{align}\label{eq:Rp}
    &R_p(s_p,a,s_p')=\nonumber\\
    &\begin{cases}
    r_n d(s,E)(1-b(e)) + r_g b(e) & \exists\sigma,\text{s.t.}\, b(q,\sigma,q')=1,\\
    r_d & \text{otherwise},
    \end{cases}
\end{align}
where $s_p=(s,q)$, $s_p'=(s',q')$, $e=(q,L(s),q')$, the numbers $r_n$ and $r_g$ satisfy $r_d<r_n<0<r_g$, $|r_n|\ll|r_d|\ll|r_g|$, the function $b$ is given in (\ref{eq:b}). The set $E$ is given by
\begin{align}\label{eq:E}
    E=\bigcup_{\substack{\sigma\in\outprops(q),\\b(q,\sigma,q')=1}}L^{-1}(\sigma).
\end{align}
The term $d(s,E)=\inf_{s''\in E}\set{d(s,s'')}$ measures the distance from the MDP state $s$ to the set $E$, where $d(s,s'')$ denotes the distance between the states $s$ and $s''$. 

The large positive number $r_g$ is used to reward taking an accepting transition or a transition that can lead to an accepting one, the small negative number $r_n$ is used to guide the transitions in the state space $S$ of the MDP $M$ to encourage the occurrence of the desired transitions between LDBA states, and the negative reward $r_d$ will be collected if the corresponding run in $\A$ hit a trap.

\subsection{The Proposed Algorithm}
The authors of  \cite{gao2019reduced} propose a method that initializes each episode with the initial Rabin state for a discrete product MDP model. The approach in \cite{sadigh2014learning} also resets the Rabin state with the initial state $q_0$ periodically. The main drawback of doing this is that we can only have a good reward if a training episode produces a trajectory that successfully reaches an accepting state in the DRA.  However, for a product MDP with continuous state and action spaces, the sampling complexity of getting such a satisfactory trajectory is too high when we have a complex LTL formula and consequently, we cannot obtain enough reward to train the neural networks for a good performance. As a result, at the beginning of each episode, we sample a $q_{\mathrm{init}}$ instead of using the initial state $q_0$ as given by the translated automaton. Then the initial state of the product MDP is constructed by  using this sampled $q_{\mathrm{init}}$. 

\begin{algorithm}[ht]
	\caption{Actor-Critic Algorithm for Continuous Product MDP $M_p$}
	\label{ddpg}
	\begin{algorithmic}[1]
	\small
		\REQUIRE labeled MDP $M$, LDBA $\A$, product MDP $M_p$
		\ENSURE Policy $\pi$
		\STATE Initialize the critic network $Q^{\omega}$, the actor network $\pi^{\theta}$ with arbitrary weights $\omega$ and $\theta$, initialize $V[i]=1$ for any $i\in\{1,..., m\}$ \;
		\STATE Copy target network $Q^{\omega'}$ and $\pi^{\theta'}$ with weights $\omega'\gets\omega$ and $\theta'\gets\theta$\;
		\STATE Initialize buffer $B$\;
		\FOR {each episode}
		\STATE Sample a state $s_0$ from $S$ in $M$ \;
		\STATE Sample a $q_{\mathrm{init}}$ from $\Q$ in $\A$ \;
		\STATE Construct initial state for $\mathcal{M}_p$ with $s_{p0}\gets(s_0,q_{\mathrm{init}})$\;
		\WHILE {$V[i]=1$ for some $i\in\{1,..., m\}$}
		\STATE Get an action $a_t$ from $\pi^{\theta}$\;
		\STATE Simulate from $s_t$ to $s_t'$ using $a_t$\;
		\STATE Get LDBA state $q_t'\gets\delta(q_t,L(s_t))$
		    \IF {{$q_{t}'$ is a \emph{trap}}}
		    \BREAK
		    \ENDIF
		\STATE Get the next state $s'_{pt}\gets(s'_t, q'_t)$
		\STATE Calculate process reward $R(s_{pt},a_t,s_{pt}')$\;
		\STATE Store tuple $\{s_{pt}, a_t, R, s_{pt}'\}$ in buffer $B$\;
		\STATE Sample $N$ batches from the buffer and calculate target values for $i\in N$ with 
		\begin{equation}
		y_i=R_i+\gamma Q^{\omega'}(s_{pi}',\pi^{\theta'}(s_{pi}'))\;
		\end{equation}
		\STATE Update the critic network by minimizing the loss function: 
		\begin{equation}
		L=\frac{1}{N}\sum_{i\in N}(y_i-Q^{\omega}(s_{pi},\pi^{\theta}(s_{pi})))^2\;
		\end{equation}
		\STATE Update the actor network according to:
		\begin{equation}
		\theta_{k+1}\gets\theta_{k}+\alpha\frac{1}{N}\sum_{i\in N}\nabla_{a_i}Q^{\omega}(s_{pi},a_i)\nabla_{\theta}\pi^{\theta}(s_{pi})\;
		\end{equation}
		\STATE Update state $s_{pt}\gets s'_{pt}$
		\STATE Update the target networks:
		\begin{equation}
		\begin{split}
		\theta'&\gets\tau\theta+(1-\tau)\theta'\\
		\omega'&\gets\tau\omega+(1-\tau)\omega'
		\end{split}
		\end{equation}
		
		\ENDWHILE
		\ENDFOR
		
	\end{algorithmic}

\end{algorithm}

We use DDPG \cite{lillicrap2015continuous} to train the neural networks. As most reinforcement learning algorithms in which data has to be independently and identically distributed, a buffer is used here for storing only the last $N$ steps of transition data \cite{mnih2013playing}. At each time step, the tuple $\{s_p, a, R_p, s_p'\}$ is stored into the buffer and a batch of data is uniformly sampled from the buffer for training the networks. As is shown in Algorithm \ref{ddpg} in line 16 and 17, the critic is updated with minimizing the loss function of the neural network and the actor is updated such that the average value is used to approximate the expectation as in Eq.~\ref{eq:policy}. It was discussed in \cite{mnih2015human} that directly implementing deep $Q$-learning with neural networks will be unstable because the $Q$ value is also used for policy network training. As a result, a small change in the $Q$ value may significantly change the policy and therefore change the data distribution. The authors proposed a way of solving this issue by cloning the $Q$-network to obtain a target network after each fixed number of updates. This modification makes the algorithm more stable compared with the standard deep $Q$-learning. We use two target networks $Q^{\omega'}(s_p,a)$ and $\pi^{\theta'}(s_p)$ as in \cite{lillicrap2015continuous}. The target networks are copied from the actor and critic networks in the beginning and the weights of both networks are updated after every several steps by using $\omega'\gets\tau\omega+(1-\tau)\omega'$ and $\theta'\gets\tau\theta+(1-\tau)\theta'$ with $\tau\ll 1$. 

The proposed method to solve a continuous MDP  with LTL specifications is summarized in Algorithm \ref{ddpg}.

\subsection{Analysis of the Algorithm}

While DDPG does not offer any convergence guarantees for approximating a general nonlinear value function, we prove in this section that, if the MDP is finite (e.g. obtained as a finite approximation of the underlying continuous-state MDP), the reward function defined by (\ref{eq:Rp}) does characterize Problem \ref{prob2} correctly in the sense that the optimal policy can satisfy the formula at each state such that the formula is satisfiable. 

\begin{thm}
 Let $\phi$ be an LTL formula and $M_p$ be the product MDP formed from the MDP $M$ and an LDBA translation $\mathcal{A}_{\phi}$ encoding $\phi$. Then there exists some $\gamma^*\in(0,1)$, $r_g$, $r_d$ and $r_n$ given in (\ref{eq:Rp}) such that for all $\gamma\in(\gamma^*,1)$ the optimal policy on $M_p$ satisfies $\mathcal{A}_{\phi}$ for each initial state $s_0\in M_p$ such that $\mathcal{A}_{\phi}$ is satisfiable. 
\end{thm}

\begin{proof}
\textit{(Sketch of proof)} Suppose that $\pi_i$ ($i=1,2$) are two policies such that $\pi_i$ has probability $p_i$ of satisfying $\mathcal{A}_{\phi}$ (i.e. producing accepting runs on $M_p$) from an initial state $s_0\in S_p$. We show that if $v_{\pi_1}(s_0)\ge v_{\pi_2}(s_0)$, then $p_1=0$ implies $p_2=0$. Suppose that this is not the case, i.e., $p_2>0$ and $p_1=0$. We have 
\begin{align*}
v_{\pi_i}(s_0) & = \mathbb{E}[G_t|s_t=s_0,\xi_t\vDash\mathcal{A}_{\phi}]\mathbb{P}_{\pi_i}(s_0\vDash\mathcal{A}_{\phi})  \\
&\qquad\qquad + \mathbb{E}[G_t|s_t=s_0,\xi_t\not\vDash\mathcal{A}_{\phi}]\mathbb{P}_{\pi_i}(s_0\not\vDash\mathcal{A}_{\phi}),
\end{align*}
where $\xi_t$ denotes a run of the product MDP under $\pi_i $ starting from $s_t$. 

By carefully estimating the accumulated reward, we can get an upper bound for $v_{\pi_1}(s_0)$ and an lower bound for $v_{\pi_2}(s_0)$ as follows: 
\begin{align*}
v_{\pi_1}(s_0) &\le  p_1 r_g \frac{1}{1-\gamma} + (1-p_1)r_g C = r_gC, \\
v_{\pi_2}(s_0) &\ge  p_2 \left(r_g \frac{\gamma^k}{1-\gamma^k} - \frac{M}{1-\gamma}\right)-(1-p_2)\frac{M}{1-\gamma}, 
\end{align*}
where $M=\max(|r_n|d_{\max},|r_d|)$, $d_{\max}$ is the maximum value that can be taken by $d(\cdot,\cdot)$ in (\ref{eq:Rp}), and $C>0$, $k>0$ are constants (depending on the product MDP). Since $p_2>0$, there exists $\gamma^*\in(0,1)$ and a choice of a sufficiently large $r_g$ (depending on $\gamma^*$ and other constants) such that $v_{\pi_2}(s_0)>v_{\pi_1}(s_0)$ for all $\gamma\in (\gamma^*,1)$. This contradicts $v_{\pi_1}(s_0)\ge v_{\pi_2}(s_0)$. 
\end{proof}

\begin{rem}
Note that this result does not offer guarantees that a policy that maximizes $v_{\pi}(s)$ for all $s$ also maximizes the satisfaction probability for all $s$. Nonetheless, we guarantee that the optimal policy always satisfies the formula, provided that the formula is satisfiable. Our formulation is consistent with that in \cite{oura2020reinforcement}. For future work, we can investigate how to integrate the reward formulation in \cite{hahn2019omega} and those in this paper to maximize satisfaction probability. 
\end{rem}

\section{Simulation Results}	
	
In this section, we test the proposed method with different LTL specifications using a car-like robot as in \cite{Li2018arxiv}: 
\begin{equation}
\begin{split}
\dot{x}&=\frac{v\cos{(\gamma+\theta)}}{\cos{\gamma}},\\
\dot{y}&=\frac{v\sin{(\gamma+\theta)}}{\cos{\gamma}},\\
\dot{\theta}&=v\tan{\phi},
\end{split}
\end{equation}
where $(x,y)$ is the planar position of center of the vehicle, $\theta$ is its orientation, the control variables $v$ and $\phi$ are the velocity and steering angle, respectively, and $\gamma=\arctan{(\tan{(\phi)})/2}$. The state space is $\mathcal{X}=[-5, 5]\times[-5, 5]\times[-\pi, \pi]$ and the control space is $\mathcal{U}=[-1,1]\times[-1,1]$. 
\subsection{Example 1}
In the first example, we test our algorithm with a simple LTL specification
\begin{equation}
\phi_1=\lozenge(a\wedge\lozenge b),
\end{equation}
where $a=[-3.5, -2]\times[-3.5, -2]$ and $b=[2,3.5]\times[2,3.5]$ are two regions in working space. This LTL formula specifies that the robot must reach $a$ first and then reach $b$. We compare our algorithm that samples a random $q$ with the standard method that resets $q$ to $q_0$ at the beginning of each episode. The neural networks are trained for 1 million steps with 200 steps in each episode. The simulation step is $\Delta t=0.1$s. We use $r_g=50$ and $r_n=-0.1$ for the reward function as in Eq. \ref{eq:Rp}. The simulation result of example 1 is presented in Fig.~\ref{fig:example1}. The areas marked as blue are the regions $a$ and $b$. We show the trajectories from an initial point at $(0,-2.5,0)$ in Fig. \ref{fig:example1a}. The black curve is the trajectory generated using the idea of fixing $q_0$ at the beginning of each episode and the red one is the trajectory from our method. It is shown that for this simple LTL specification, both ideas provide a successful trajectory. Fig. \ref{fig:example1b} shows the normalized reward during training for both ideas. The blue one is the normalized reward for the standard method and the red curve is for our method. Our method collects a normalized reward of $-0.1$ for 500k steps of training and $0.2$ for 1M steps of training while the standard method obtains a normalized reward of $-0.3$ and $0.1$ for 500k steps and 1M steps, respectively. The runtime of both methods are the same because this is determined by the number of training steps. 
The success rate of both methods are presented in TABLE \ref{table:1}. We can see that our method achieves better performance in the same amount of time as the standard method.

\begin{table}
\centering
\begin{tabular}{@{}llc@{}}
\toprule
     & \phantom{abc} & Success rate \\ 
\midrule
Standard method & & 76.7\%          \\ 
Our method      & & 83.3\%          \\
\bottomrule
\end{tabular}
\caption{Success rate for example 1 with 30 initial states.}
\label{table:1}
\end{table}

\begin{figure}
	\centering
	\begin{subfigure}[h]{0.4\linewidth}
		\includegraphics[width=1\linewidth]{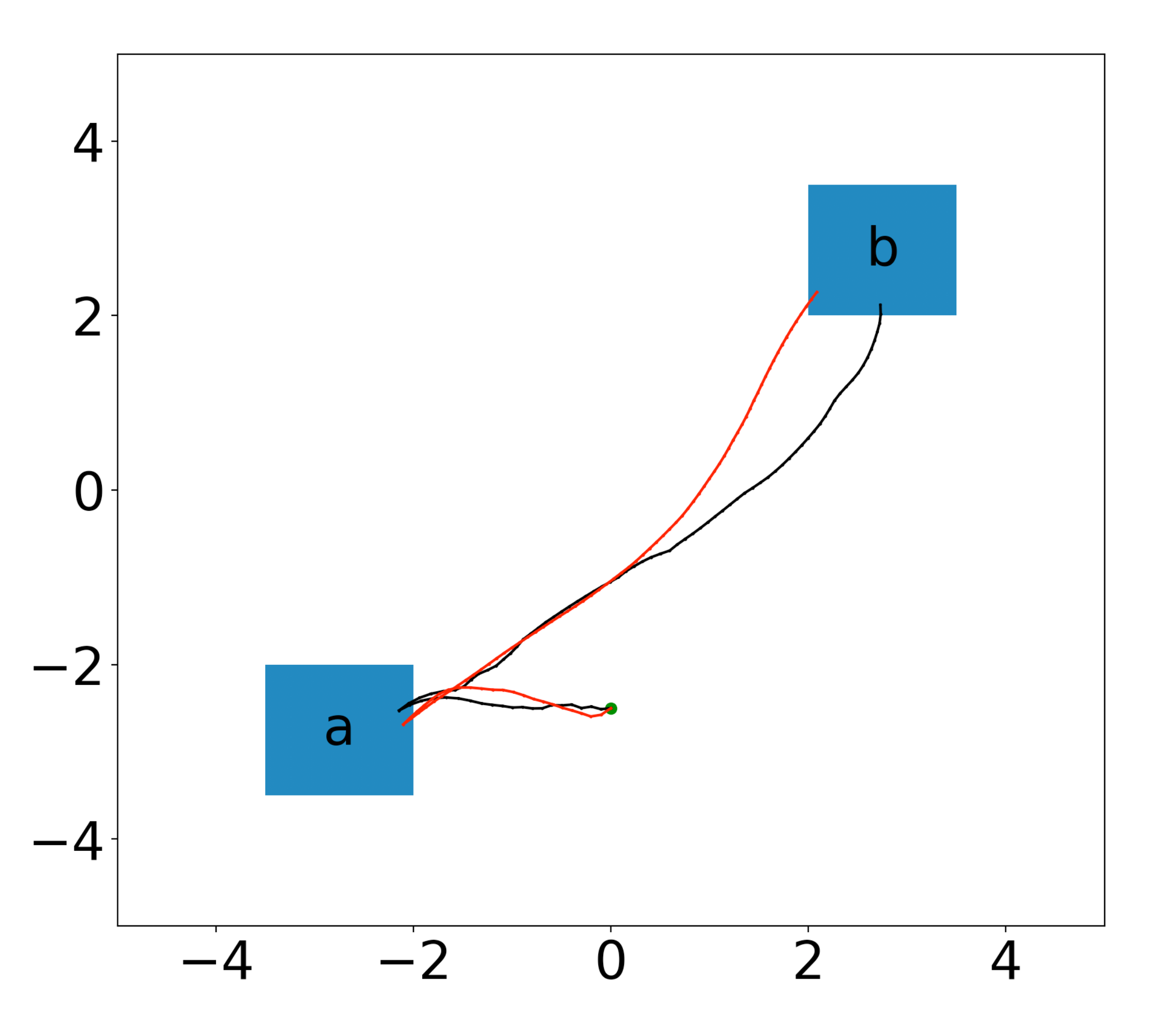}
	\caption{}
	\label{fig:example1a}
	\end{subfigure}
	\begin{subfigure}[h]{0.6\linewidth}
		\includegraphics[width=1\linewidth]{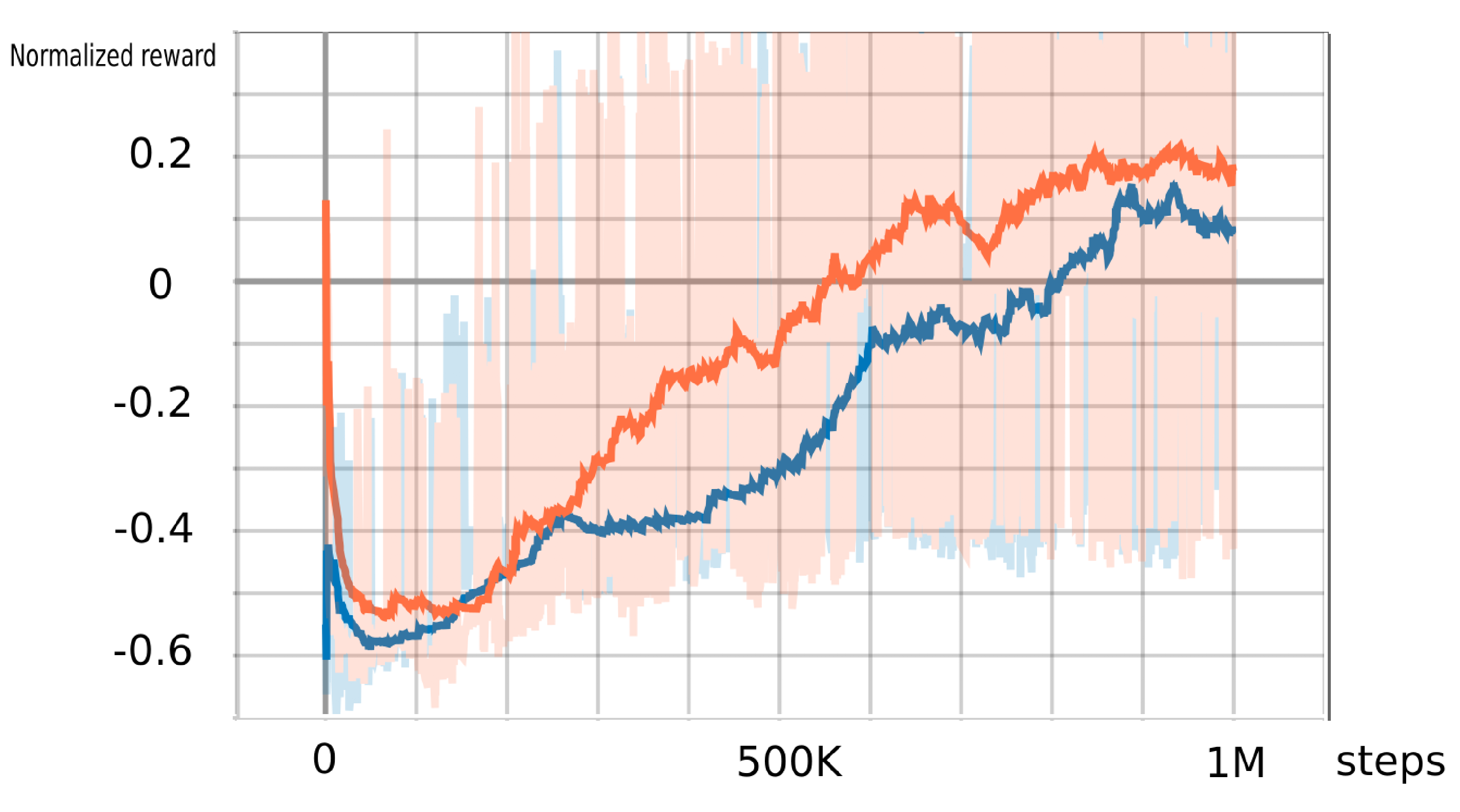}
	\caption{}
	\label{fig:example1b}
	\end{subfigure}
	
	\caption{Simulation and training results of Example 1: (a) Simulated trajectories from the initial point $(0,-2.5,0)$ under the trained policies. (b) Normalized reward during training: the blue curve and region are the normalized reward and variance of the standard method and the red curve and region are the normalized reward and variance of our method.}
	\label{fig:example1}
\end{figure}

\subsection{Example 2}
In the second example, we test our algorithm using following LTL specification:
\begin{equation}
\phi_2=\lozenge(a\wedge\lozenge(b\wedge\lozenge(c \wedge\lozenge d))),
\end{equation}
where $a=[-3,-1.5]\times[-3,-1.5]$, $b=[-3,-1.5]\times[1.5,3]$, $c=[2,3.5]\times[1.5,3]$ and $d=[2,3.5]\times[-3,-1.5]$ are four areas in the working space. In other words, we want the robot to visit $a$, $b$, $c$, $d$ sequentially. 

We train the neural networks for 1 million steps. The system is also simulated using a time step $\Delta t=0.1$s. The reward function is the same as in the first example, where $r_n=-0.1$ and $r_g=50$.  For the standard method, there are 600 steps in each episode. We increased the number of steps in each episode so that a trajectory will be long enough to satisfy the whole LTL specification. In our method, we still have 200 steps in each episode. The trajectory generated from our method is shown in Fig. \ref{fig:example2}. The success rate of both methods are presented in TABLE~\ref{table:2}.

\begin{table}
\centering
\begin{tabular}{@{}llc@{}}
\toprule
     & \phantom{abc} & Success rate \\ 
\midrule
Standard method & & 13.3\%          \\ 
Our method      & & 63.3\%          \\ 
\bottomrule
\end{tabular}
\caption{Success rate for example 2 with 30 initial states.}
\label{table:2}
\end{table}

\begin{figure}
	\centering
	\begin{subfigure}[h]{0.5\linewidth}
		\includegraphics[width=1\linewidth]{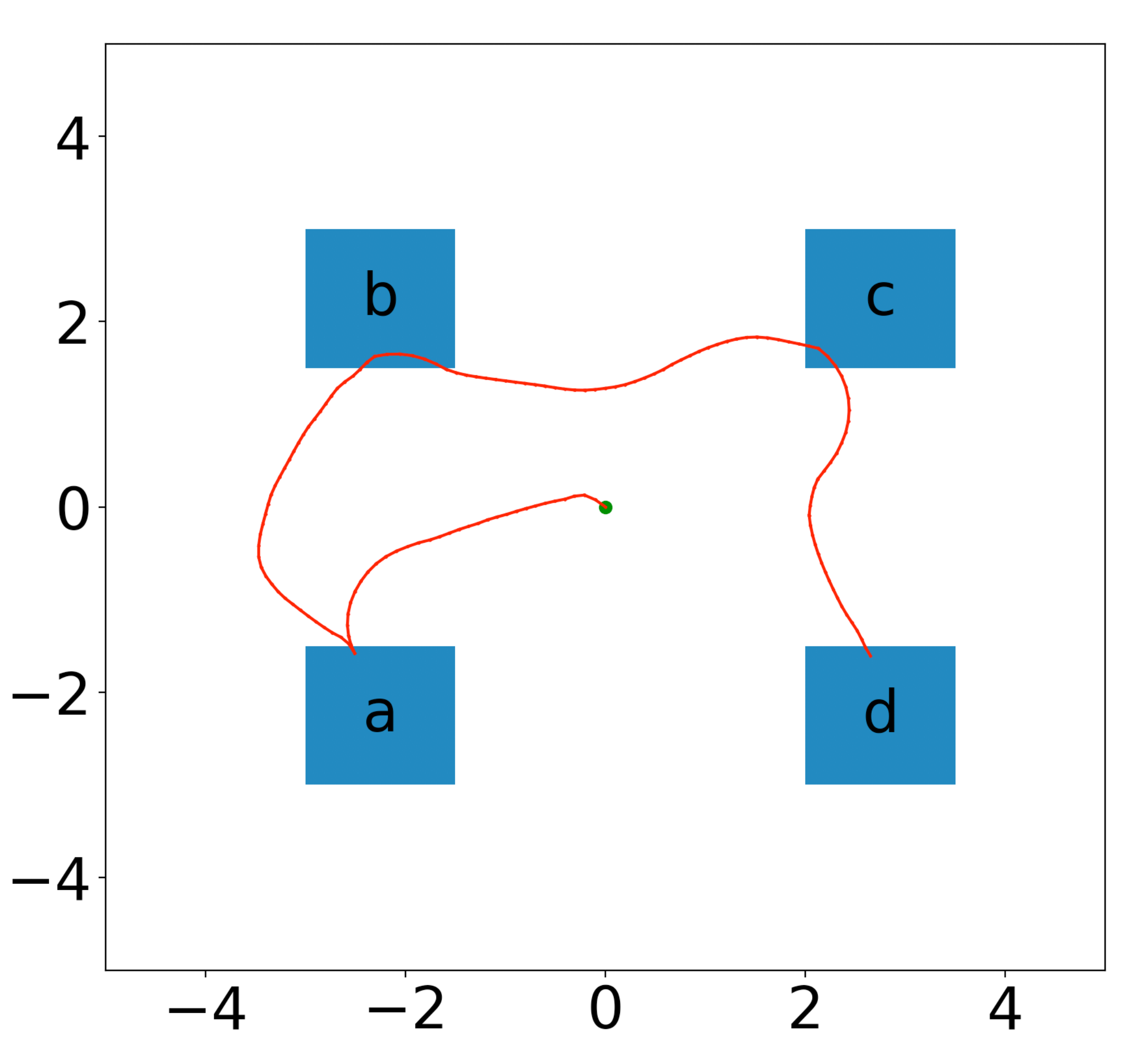}
    \caption{}
	\end{subfigure}

	\caption{Simulation results of Example 2 of our method: $a$, $b$, $c$ and $d$ are four regions marked as blue squares. The initial position $(0,0,0)$ is marked by the green spot in both figures. The red curve is the trajectory.}
	\label{fig:example2}
\end{figure}

\subsection{Example 3}

In the third example, we test our algorithm using the following LTL specification:
\begin{equation}
\phi_3=\lozenge(a\wedge\lozenge d)\vee\lozenge(b \wedge(\neg c \until d)).
\end{equation}
In plain words, the specification encodes that the robot must reach either $a$ or $b$ first. If it reaches $a$ first, then it must next reach $d$ without any other restrictions.  If it reaches $b$ first, then it has to reach $d$ without entering $c$.  We consider this specification with two different layouts of the regions $a$, $b$, $c$ and $d$.
\subsubsection{Case 1}
In case 1, $a=[-4, -3]\times[-3, -2]$, $b=[-4, -3]\times[1, 2]$ and $d=[3, 4.5]\times[1.5, 3]$ are three goals marked as blue and $c=[-1, 1.5]\times[-1, 3.5]$ is a restricted area in the working space marked as yellow as in Fig. \ref{fig:example3a} and Fig. \ref{fig:example3b}. 

The reward function is of $r_n=-0.1$ and $r_g=100$. Since we have a constraint of not entering region $c$ if it reaches $b$, we assign $r_d=-10$ if this happens. We also train the networks for 1 million steps with 200 steps in each episode. The simulation time step is $\Delta t=0.1$s. The trajectories generated from two initial states $(-2,1.5,0)$ and $(-2,-4,0)$ are shown in Fig. \ref{fig:example3a} and Fig. \ref{fig:example3b}. For initial point at $(-2,1.5,0)$, it is closer to region $b$ so that the trajectory reaches $b$ first. According to the LTL specification, it has to avoid $c$ before reaching $d$ if it reaches $b$ first. For initial point at $(-2,-4,0)$, the trajectory first reaches region $a$ and then, the trajectory can reach $d$ without avoiding $c$. The simulation results show that our method can successfully generate a policy that satisfies the LTL specification for different initial points. The algorithm learns that the trajectory should choose the target that is closer to the initial point between $a$ and $b$ and then reach $d$ according to the specification.

\subsubsection{Case 2}
In Case 2, $a$, $b$, $d$ are the same regions as in Case 1. Region $c$ is the area of $C=[-4.5,-2.5]\times[0,3]$. As in Fig. \ref{fig:example3c}, $a$, $b$, $d$ are marked as blue and $c$ is the yellow area plus the area of region $b$. As is shown in the figure, $b$ is enclosed by $c$, which means that if a trajectory enters $b$, it will also be in $c$. This implies that the automaton will be trapped in the deadlock between $q=0$ and $q=3$ and will never reach accepting condition $q=2$ as shown in Fig. \ref{fig:ltl2}. The learning algorithm is able to figure out that even if the initial point is closer to $b$, it still need to reach $a$ first. The result is shown in Fig. \ref{fig:example3c}. The success rate for both case 1 and 2 as in TABLE~\ref{table:3}
\begin{figure}
	\centering
	\begin{subfigure}[h]{0.5\linewidth}
		\includegraphics[width=1\linewidth]{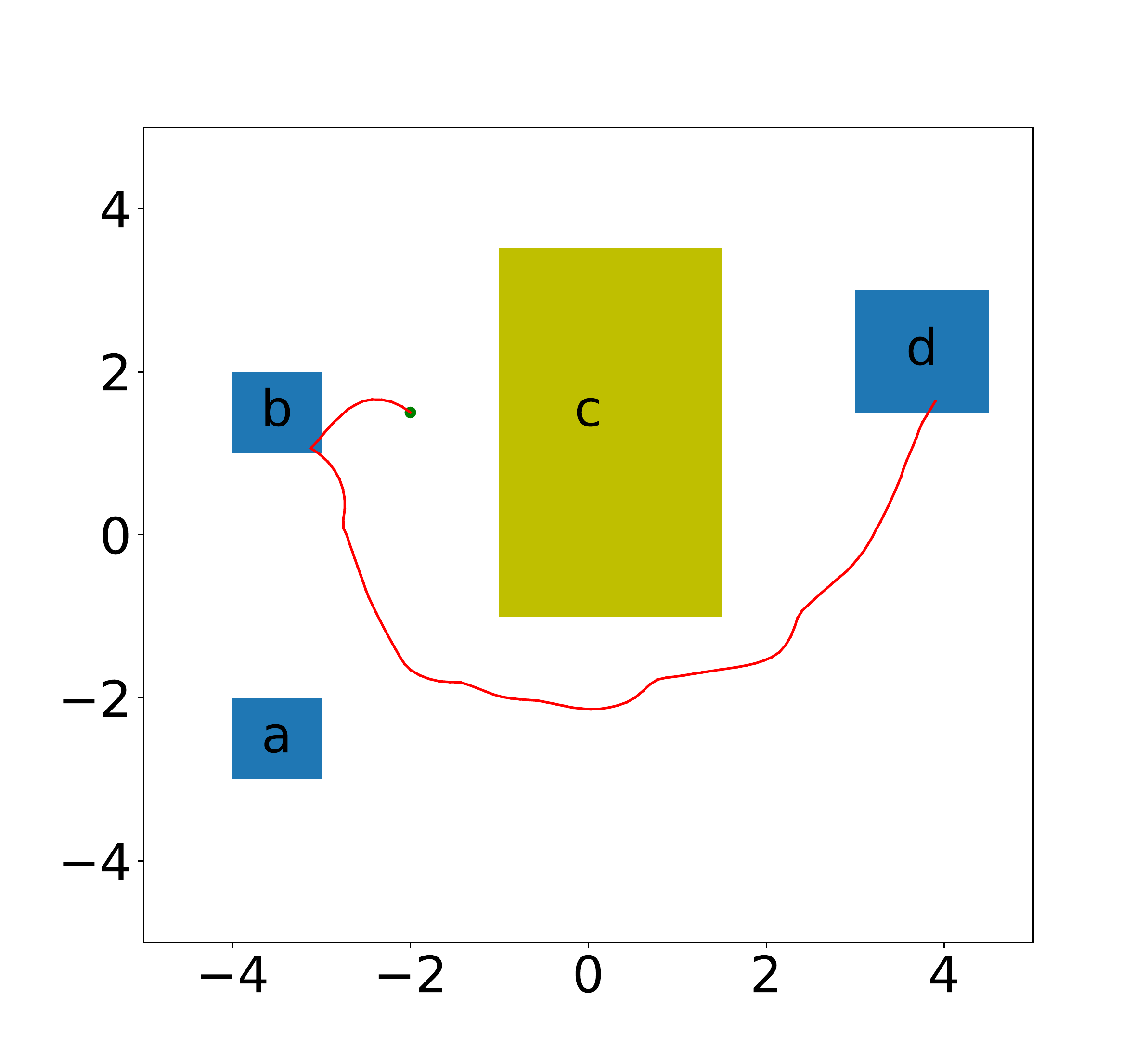}
	\caption{}
	\label{fig:example3a}
	\end{subfigure}
	\begin{subfigure}[h]{0.48\linewidth}
		\includegraphics[width=1\linewidth]{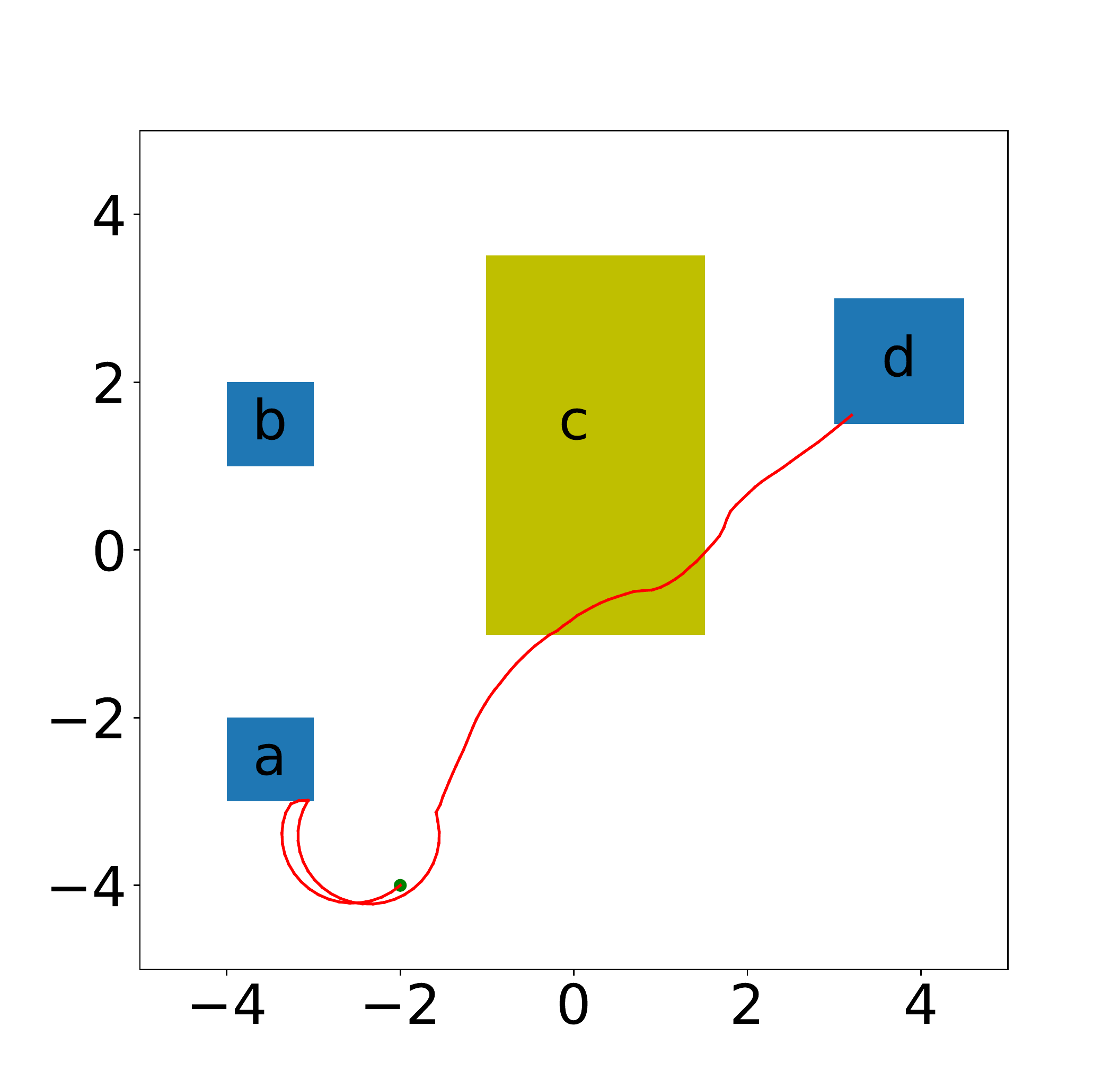}
	\caption{}
	\label{fig:example3b}
	\end{subfigure}
	\newline
    \centering
	\begin{subfigure}[h]{0.5\linewidth}
		\includegraphics[width=1\linewidth]{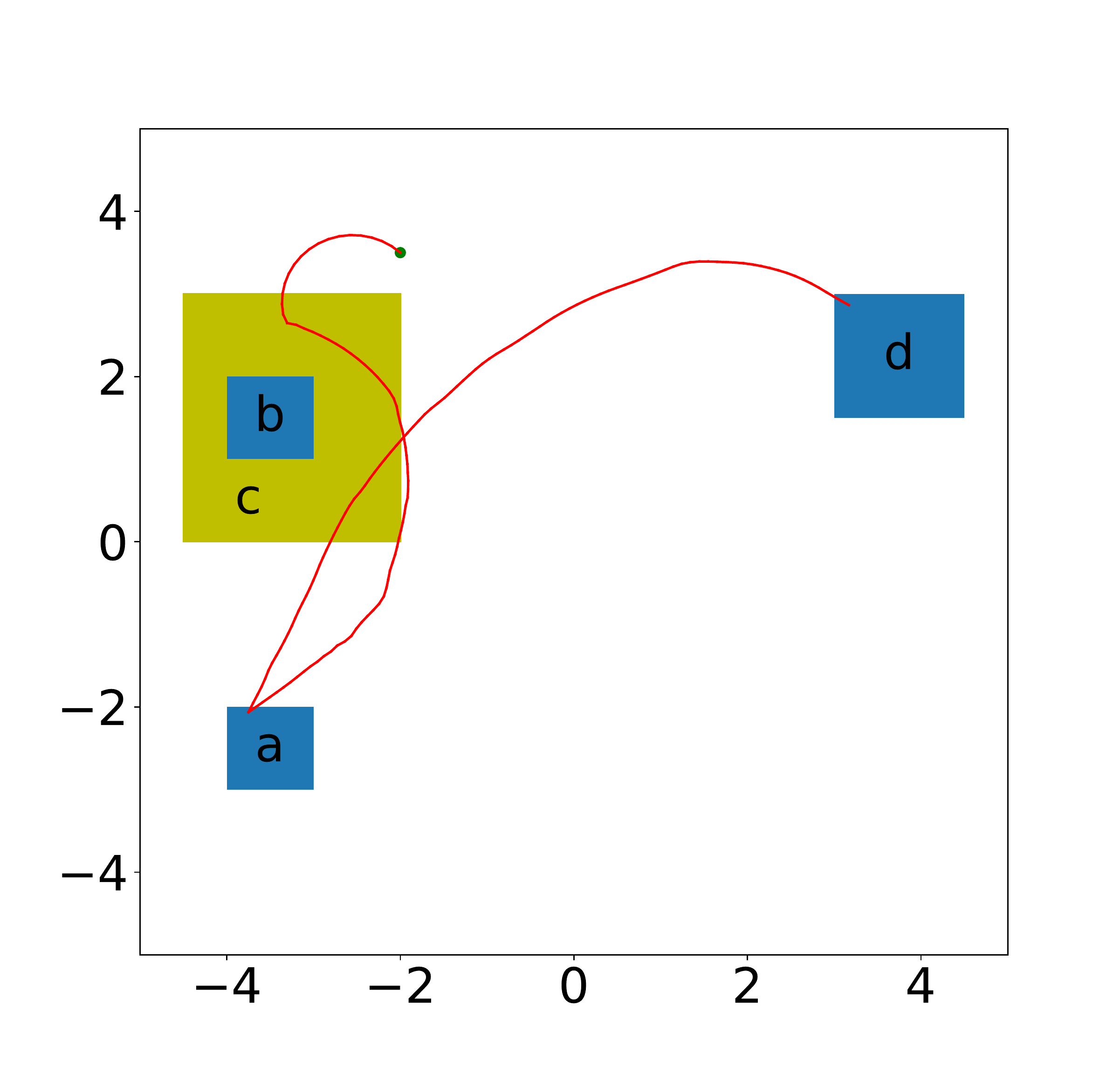}
	\caption{}
	\label{fig:example3c}
	\end{subfigure}

	\caption{Simulation results of Example 3:  The initial positions are marked as green spot and the red curves are trajectories. (a): Case 1 with initial position at $(-2,1.5,0)$. (b): Case 1 with initial position at $(-2,-4,0)$. (c): Case 2 with initial position at $(-2, 3.5,0)$.}
	\label{fig:example3}
\end{figure}

\begin{table}
\begin{tabular}{@{}lcc@{}}
\toprule
& Success rate (case 1) & Success rate (case 2) \\ 
\midrule
Standard method & 10\%                     & 6.7\%                   \\ 
Our method      & 63.3\%                   & 60\%                     \\ 
\bottomrule
\end{tabular}
\caption{Success rate of example 3 with 30 initial states.}
\label{table:3}
\end{table}

\section{CONCLUSIONS}

In this paper, we proposed a learning method for motion planning problems with LTL specifications with continuous state and action spaces. The LTL specification is converted into an annotated LDBA and the deep deterministic policy gradient method is used to train the resulting product MDP. The annotated LDBA is used to continuously shape the reward so that dense reward is available for training. We sample a state from the annotated LDBA at the beginning of each episode in training. We use a car-like robot to test our algorithm with three LTL specifications from different working configurations and initial positions in our simulation. Simulation results show that our method achieves successful trajectories for each of the specifications. For future work, we found out in our simulation that the algorithm sometimes fails to deal with complex working configurations such as non-convex obstacles. We will focus on doing research about improving the algorithm to deal with more complex scenarios and LTL specifications.

\bibliography{root}{}
\bibliographystyle{plain}

\end{document}